\documentclass{article}
\usepackage{graphicx}
\usepackage[numbers,sort&compress]{natbib}
\usepackage[preprint]{neurips_2026}
\newcommand{\ourmethod}{$\mu$-GRPO}
\usepackage{wrapfig}
\usepackage{amsmath,amssymb,amsthm,mathtools}
\usepackage{enumitem}
\usepackage[utf8]{inputenc} 
\usepackage[T1]{fontenc}    
\usepackage{hyperref}       
\usepackage{url}            
\usepackage{booktabs}       
\usepackage{amsfonts}       
\usepackage{nicefrac}       
\usepackage{microtype}      
\usepackage{xcolor}         
\usepackage[capitalize,noabbrev]{cleveref}
\usepackage{colortbl}
\usepackage{multirow}
\usepackage[most]{tcolorbox}
\usepackage{subcaption}

\newtheorem{theorem}{Theorem}
\newtheorem{corollary}{Corollary}
\newtheorem{remark}{Remark}

\definecolor{plotblue}{RGB}{70, 114, 196}
\definecolor{plotyellow}{RGB}{230, 175, 45}

\title{How Off-Policy Can GRPO Be? \\ \boldmath $\mu$-GRPO for Efficient LLM Reinforcement Learning}

\author{
  Minghao Tian \quad Yunfei Xie \quad Chen Wei \\
  Rice University
}

\begin{document}
\maketitle

\begin{abstract}
Group Relative Policy Optimization (GRPO) has been a key driver of recent progress in reinforcement learning with verifiable rewards (RLVR) for large language models, but it is typically trained in a low-staleness, near-on-policy regime that incurs substantial system overhead. We ask a simple question: How off-policy can GRPO be? We show that GRPO-style algorithms can tolerate substantially larger rollout staleness than previously assumed, and propose $\mu$-GRPO, an RL training framework that organizes training into a small number (\textit{e.g.}, four) of large sequential generation--optimization stages. This design induces high rollout staleness while greatly reducing rollout--optimization switching overhead. To stabilize learning under stale data, $\mu$-GRPO combines relaxed clipping, which preserves useful stale-rollout gradients, with negative-advantage veto, which removes destabilizing post-trigger suffix updates in negative-advantage responses. Across five language models and multiple math reasoning benchmarks, $\mu$-GRPO matches or exceeds the performance of standard GRPO while achieving around 2$\times$ speedup in wall-clock training time, establishing a substantially improved performance--efficiency trade-off for LLM reinforcement learning. 

\medskip
\noindent\textbf{Project page:}
\href{https://benjamint2048.github.io/mu-grpo-project-page/}
{\texttt{benjamint2048.github.io/mu-grpo-project-page/}}
\end{abstract}
\section{Introduction}
\label{sec:introduction}

Reinforcement Learning with Verifiable Rewards (RLVR) has rapidly reshaped the landscape of large language model (LLM) reasoning, driving landmark systems such as OpenAI o1~\citep{openai2024openaio1card} and DeepSeek R1~\citep{Guo_2025}. A dominant narrative attributes the success of these RL approaches, most notably Group Relative Policy Optimization (GRPO)~\citep{schulman2017proximalpolicyoptimizationalgorithms, shao2024deepseekmathpushinglimitsmathematical, yu2026dapo}, to their on-policy design, where rollouts are freshly generated and closely aligned with the current policy. This freshness is widely viewed as important for stable credit assignment and iterative self-improvement.

On closer inspection, however, standard GRPO implementations are almost never strictly on-policy. The off-policyness arises from multiple sources. One primary source is data staleness: most RL training frameworks~\citep{sheng2025hybridflow, slime_github, noukhovitch2025asynchronous, zhong2025streamrl, he2025history} use rollouts generated by policies from several updates earlier, resulting in a \emph{near} on-policy regime that trades off rollout freshness for efficiency~\citep{openr1}. Throughout this paper, we refer to this commonly used low-staleness, near-on-policy implementation as \emph{standard GRPO}. Another source of off-policyness stems from training–inference mismatch. In practice, rollouts are typically generated using high-throughput inference backends like vLLM~\citep{kwon2023efficient}, while policy optimization is performed using sharded training engines like FSDP~\citep{zhao2023pytorchfsdpexperiencesscaling}. This introduces numerical discrepancies between rollout generation and optimization. Empirically, however, LLM RL has proven tolerant to such off-policyness when appropriate stabilization techniques are applied~\citep{liu-li-2025-rl-collapse,qi2025defeating}.

Taken together, these observations raise a natural question: \textit{How far can we increase rollout staleness in GRPO without compromising learning stability or reasoning performance?}

This question is also central to system efficiency. Standard GRPO implementations repeatedly alternate between rollout generation and policy optimization, requiring many weight synchronizations between inference and training engines. To make such switching practical, systems often keep training state resident during rollout generation, reducing memory available for high-throughput inference\citep{sheng2025hybridflow, slime_github}. Increasing rollout staleness could amortize rollout generation over many more policy updates and thereby reduce these overheads. While prior work has explored limited staleness, these approaches do not push staleness far enough to fully decouple rollout generation from optimization~\cite{zheng2026prosperity,xi2026bapo}.

\begin{wrapfigure}{r}{0.57\linewidth}
\centering
\includegraphics[width=\linewidth]{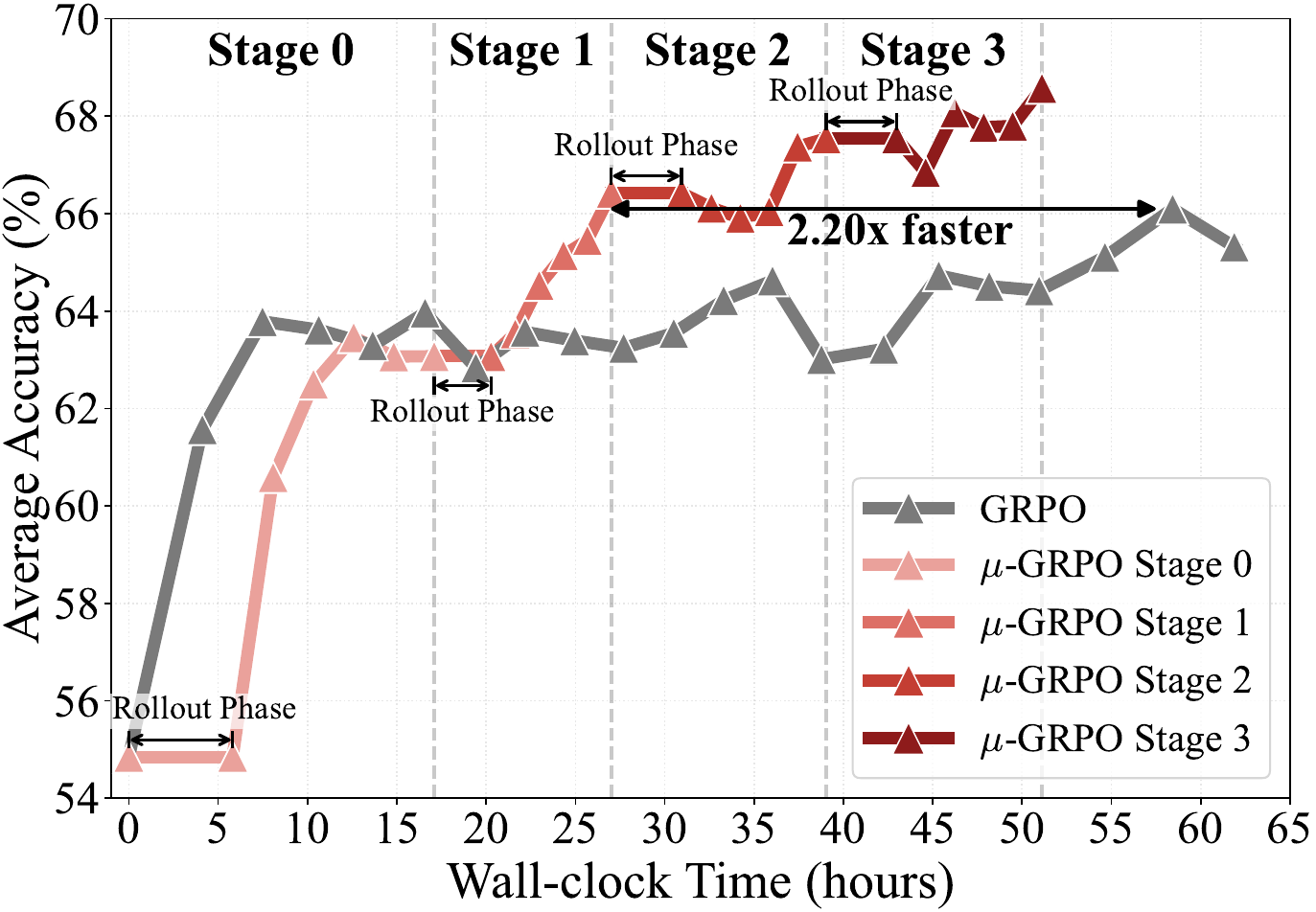}
\vspace{-15pt}
\caption{\textbf{Comparing GRPO and \boldmath \ourmethod.} Average accuracy across five math benchmarks over wall-clock time. \ourmethod{} uses four large rollout--optimization stages, inducing high rollout staleness while reducing rollout--training switching overhead. It reaches GRPO's performance with a 2.2$\times$ wall-clock speedup on DeepSeek-7B~\citep{shao2024deepseekmathpushinglimitsmathematical}. Both methods are trained for 4096 updates.}
\label{fig:teaser}
\vspace{-14pt}
\end{wrapfigure}

In this paper, we show that GRPO-style algorithms can tolerate substantially larger rollout staleness than commonly assumed, but only if the failure mode of high-staleness optimization is handled explicitly. Our diagnosis reveals that stale rollouts retain useful learning signal: relaxed clipping recovers strong early learning under large staleness, but later introduces a localized instability rather than generic off-policy collapse. We trace it to negative-advantage trajectories that cross an off-support boundary: once a prefix becomes poorly supported by the current policy, token-level GRPO can still apply substantial updates to later suffix tokens. This prefix-support mismatch explains why token-level clipping alone fails under large staleness.

Motivated by this diagnosis, we propose \textbf{\boldmath{$\mu$}-GRPO}, a staged GRPO framework for high-staleness training. Instead of frequent small generation--optimization cycles, $\mu$-GRPO uses a few large stages that execute rollout generation and policy optimization sequentially: each stage first freezes the current policy to generate a large static rollout dataset, and then consumes this dataset over many policy updates using stored behavior log-probabilities as anchors, resulting in high rollout staleness. To stabilize optimization under extreme staleness, \ourmethod{} combines relaxed clipping with negative-advantage veto, which filters destabilizing post-trigger suffix updates in negative-advantage responses while preserving useful stale-rollout gradients.

 Empirically, $\mu$-GRPO converts GRPO's staleness tolerance into substantial wall-clock speedups. As illustrated in \cref{fig:teaser}, $\mu$-GRPO completes training in a few stages (e.g., four) and exceeds the performance of the standard GRPO while achieving a 2.2$\times$ wall-clock speedup on DeepSeek-7B\footnote{DeepSeek-7B denotes the Qwen2.5-Math-7B model distilled from DeepSeek-R1. \href{https://huggingface.co/deepseek-ai/DeepSeek-R1-Distill-Qwen-7B}{link to Hugging Face models}}~\cite{shao2024deepseekmathpushinglimitsmathematical}. Across five language models and a diverse set of math reasoning benchmarks, \ourmethod \ matches or improves average accuracy of standard GRPO on four out of five models. At the same time, it reduces rollout generation time by 1.82$\times$ and overall wall-clock training time by 1.53$\times$ on average. These results show that high-staleness GRPO is not only viable, but can improve the performance--efficiency trade-off of scalable LLM reinforcement learning.
\section{Preliminaries: GRPO and Rollout Staleness}
\label{sec:preliminaries}
We briefly review Group Relative Policy Optimization (GRPO)~\citep{shao2024deepseekmathpushinglimitsmathematical} and define rollout staleness. Let $\beta$ denote the behavior policy that generates rollouts and $\pi_\theta$ the policy being optimized.

\paragraph{GRPO objective.} For each prompt $x\sim\mathcal D$, the behavior policy $\beta$ samples a group of $G$ responses $\{y_i\}_{i=1}^G$, where $y_i=(a_{i,1},\ldots,a_{i,T_i})$. Each response receives a scalar reward $R_i$. GRPO estimates advantages by normalizing rewards within the group:
\begin{equation}
    A_i
    =
    \frac{
        R_i - \operatorname{mean}(\{R_1,\ldots,R_G\})
    }{
        \operatorname{std}(\{R_1,\ldots,R_G\})
    } .
    \label{eq:grpo-adv}
\end{equation}
The same response-level advantage $A_i$ is assigned to all tokens in response $y_i$.

For token $a_{i,t}$, with prefix state $s_{i,t}=(x,a_{i,<t})$, define the token-level importance ratio
$
    \rho_{i,t}(\theta)
    =\pi_\theta(a_{i,t}\mid s_{i,t})/\beta(a_{i,t}\mid s_{i,t}).
    \label{eq:token-ratio-prelim}
$ GRPO optimizes
\begin{equation}
\begin{aligned}
    \mathcal J_{\mathrm{GRPO}}(\theta)
    =
    \mathbb E_{x\sim\mathcal D,\; \{y_i\}\sim\beta}
    \left[
    \frac{1}{G}\sum_{i=1}^G
    \frac{1}{T_i}\sum_{t=1}^{T_i}
    \min\left(
        \rho_{i,t}(\theta) A_i,\,
        \operatorname{clip}
        \big(\rho_{i,t}(\theta),1-\epsilon,1+\epsilon\big) A_i
    \right)
    \right].
\end{aligned}
\label{eq:grpo-obj}
\end{equation}
The implemented training loss is $-\mathcal J_{\mathrm{GRPO}}$. A KL term $-\lambda_{\mathrm{KL}} D_{\mathrm{KL}}\!\left(\pi_\theta \,\|\, \pi_{\mathrm{ref}} \right)$ is optionally applied.

\paragraph{Rollout staleness $\mu$.} GRPO training alternates between rollout generation and policy optimization. At the beginning of each cycle, the current policy is frozen as the behavior policy $\beta$, which generates a fixed rollout batch used for several policy updates before the next refresh.

Let $B_{\mathrm{train}}$ denote the number of prompt groups generated in one rollout phase, where each prompt group contains $G$ sampled responses. Let $B_{\mathrm{mini}}$ denote the number of prompt groups consumed by one policy update. Therefore, a single rollout generation phase supports $\mu = {B_{\rm train}}/{B_{\rm mini}}$ model updates, which we refer to as the \textit{rollout staleness}.

During these $\mu$ updates, $\theta$ changes while the behavior policy $\beta$ and the stored rollout log-probabilities remain fixed. Thus later updates within the same cycle optimize against increasingly stale data. A strictly on-policy algorithm corresponds to $\mu=1$, where every update uses freshly generated rollouts. 

GRPO is therefore a \emph{near on-policy} algorithm, with $\mu$ typically set to a small value (e.g., $\mu\,$=$\,4$~\citep{zheng2026act,wang2026reinforcement,yu2026dapo,chen2025minimax,zheng2025group}) to control policy drift under PPO-style clipped updates. While it is generally understood that a large $\mu$ degrades performance,\textbf{ we show in this work that substantially larger values of \boldmath $\mu$ are feasible}, achieving comparable or improved performance while accelerating training.

Throughout the paper, one \emph{model update} refers to one optimizer update on one minibatch, and \emph{batch size} refers to $B_{\mathrm{mini}}$, the number of prompt groups for one optimizer update, unless stated otherwise.
\section{Diagnosing GRPO's Staleness Tolerance}
\label{sec:diagnosis}

We pursue large-$\mu$ training because it offers a compelling promise: accelerating training through extensive data reuse, turning rollout staleness from a constraint into a source of efficiency. Doing so, however, requires understanding when stale rollouts remain useful and when they become destabilizing. Rather than treating high staleness as uniformly harmful, we show that stale rollouts contain exploitable learning signal whose failure mode is localized to a specific class of off-support negative updates. This section identifies why naive high-staleness GRPO breaks (\cref{subsec:clipping_dilemma}), then identify low-ratio negative-advantage
tokens as collapse triggers (\cref{subsec:negative_drift}), and finally localize
the destabilizing updates to the post-trigger suffix
(\cref{subsec:localizing_harmful_updates}).

\subsection{The Clipping Dilemma under High Staleness}
\label{subsec:clipping_dilemma}

\begin{wrapfigure}{r}{0.65\textwidth}
\vspace{-1.0em}
\centering
\includegraphics[width=0.32\textwidth]{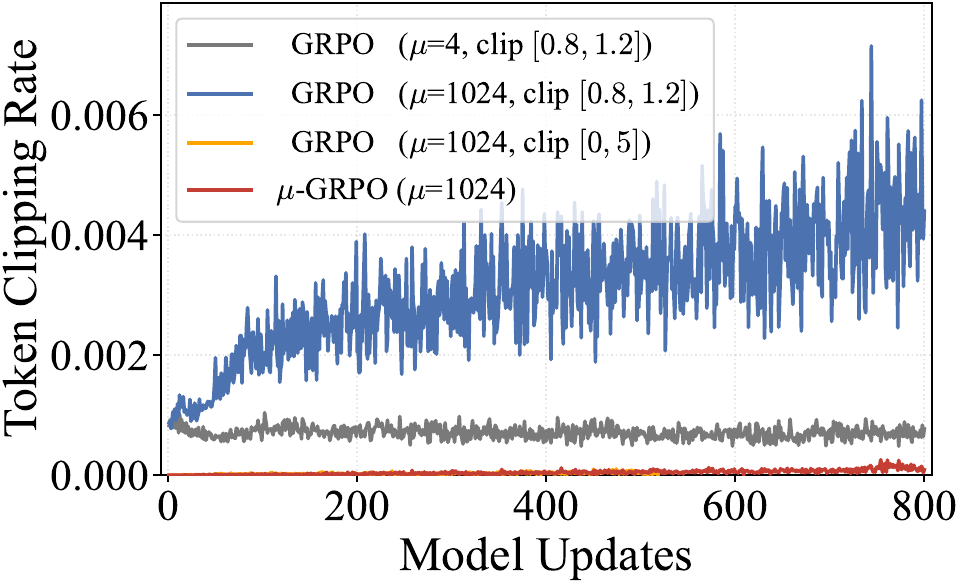}
\hfill
\includegraphics[width=0.32\textwidth]{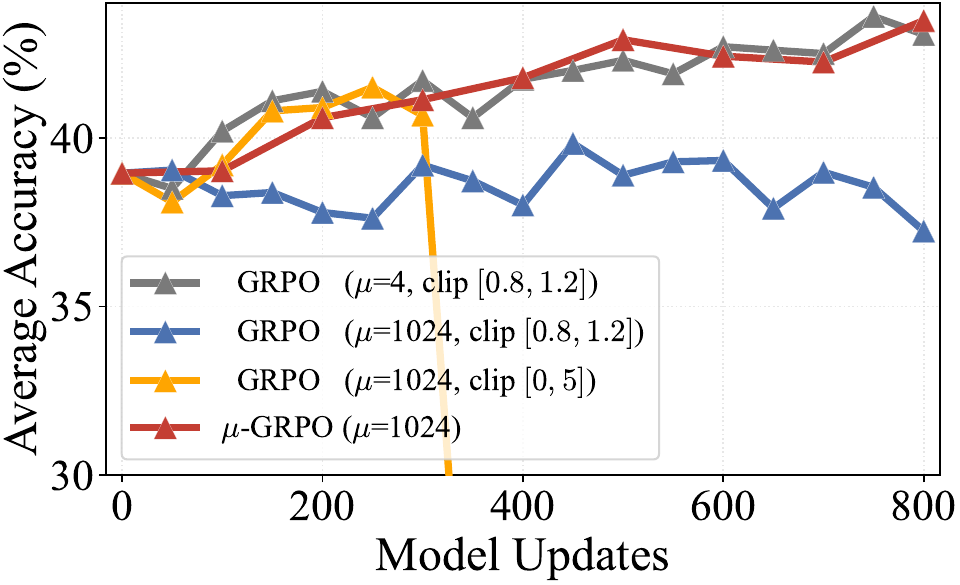}
\vspace{-1.5em}
\caption{\textbf{Clipping creates a high-staleness dilemma.}
At $\mu=1024$, $[0.8,1.2]$ bound clips many more tokens and plateaus at lower accuracy ({\color{plotblue}blue}), while relaxed clipping recovers early gains but later collapses ({\color{plotyellow}yellow}).
Results are on Qwen2.5-Math-7B.}
\label{fig:grpo_clipping_and_performance}
\vspace{-1em}
\end{wrapfigure}

Under low staleness, the behavior policy $\beta$ and current policy $\pi_\theta$ remain close, so the standard $[0.8,1.2]$ interval clips only outlying updates. Under high staleness, however, rollouts sampled from the same $\beta$ are reused for many updates, causing the current policy $\pi_\theta$ to drift away from the rollout distribution. We study this clipping dilemma at a high staleness $\mu=1024$ in Figure~\ref{fig:grpo_clipping_and_performance}. With the standard clipping bound $[0.8,1.2]$, an increasing fraction of rollout tokens fall outside the interval and are clipped, suppressing their gradient contribution (Figure~\ref{fig:grpo_clipping_and_performance}, left). As a result, the run plateaus at a lower accuracy (Figure~\ref{fig:grpo_clipping_and_performance}, right). In contrast, relaxing the clipping range to $[0,5]$ keeps the clipping ratio low and recovers strong early learning, indicating that stale rollouts still contain useful gradient signal. However, this signal is not continuously safe to use: after the initial improvement phase, the relaxed-clipping run collapses mid-training.
These results show that high-staleness GRPO is limited not by the absence of useful signal in stale rollouts, but by the difficulty of exposing that signal while suppressing destabilizing updates as $\pi_\theta$ drifts away from the behavior policy $\beta$.

\subsection{Negative-Advantage Drift as a Collapse Signal}
\label{subsec:negative_drift}

The clipping dilemma in Figure~\ref{fig:grpo_clipping_and_performance} shows that neither extreme is satisfactory under high staleness. Tight clipping is stable but clips away much of the stale rollout signal, limiting learning. Relaxed clipping exposes this signal and enables strong early improvement, but eventually collapses. We therefore ask which part of the relaxed update becomes dangerous.

We focus on negative-advantage tokens. For GRPO, a negative advantage means that the response scored worse than other responses in the same group, so the update decreases the probability of the tokens that produced it. Under high staleness, many successive updates use rollouts from the same frozen behavior policy $\beta$. With relaxed clipping, the policy can therefore repeatedly suppress negative-advantage actions sampled from this fixed distribution.

This repeated suppression creates a simple diagnostic. The importance ratio $\rho_t=\pi_\theta(a_t\mid s_t)/\beta(a_t\mid s_t)$ compares how likely a token is under the current policy and under the behavior policy that generated it. A small ratio on negative-advantage tokens means that actions once plausible under $\beta$ have become much less likely under $\pi_\theta$. The blue and gray lines in Figure~\ref{fig:nav_explanation}  show that this is exactly what happens in collapsed runs: $\mathbb{E}[\rho_t\mid A_t<0]$ drops sharply, while stable runs keep it close to
one. Collapse is therefore preceded by the over-suppression of low-reward behavior-sampled responses.

Formally, for a response $y$\,$=$\,$(a_1,\ldots,a_T)$, we omit the response index $i$. With the relaxed clipping range $[0,5]$, the effective gradient for a negative-advantage token is
$
    \nabla_\theta \ell_t(\theta)
    =
    |A_t|\,\rho_t(\theta)
    \nabla_\theta \log \pi_\theta(a_t\,|\, s_t).
$
Thus, relaxed clipping allows negative-advantage updates to keep suppressing behavior-sampled actions over a much wider ratio range.

This motivates using low-ratio negative-advantage tokens as collapse triggers. Prior work on training--inference mismatch similarly captures low-ratio events to mask out unstable updates~\cite{liu-li-2025-rl-collapse}. Here, we treat such events as boundaries where useful stale learning begins to turn into destructive policy drift, and then ask which subsequent updates must be masked to prevent collapse.

\begin{figure*}[!t]
\centering
\includegraphics[width=0.49\textwidth]{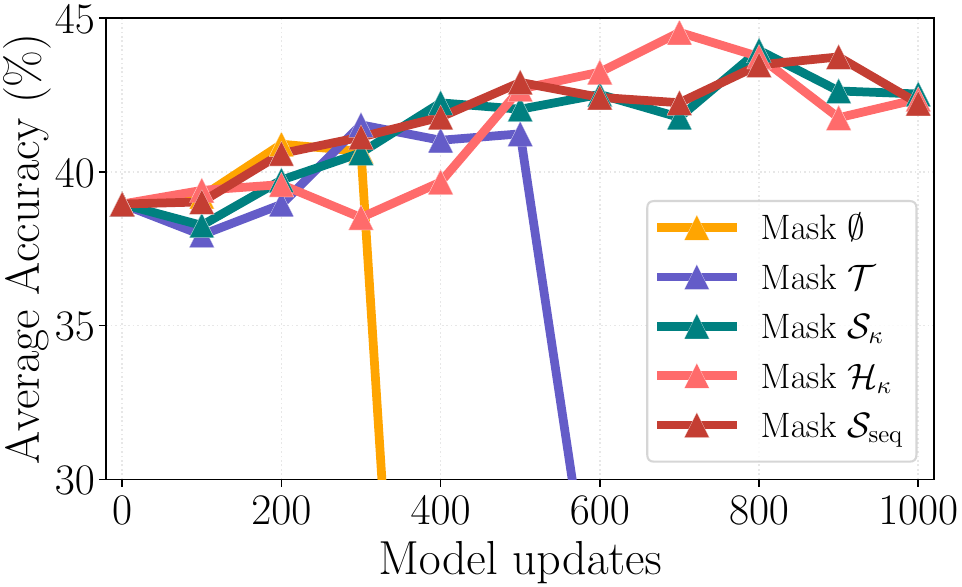}
\hfill
\includegraphics[width=0.49\textwidth]{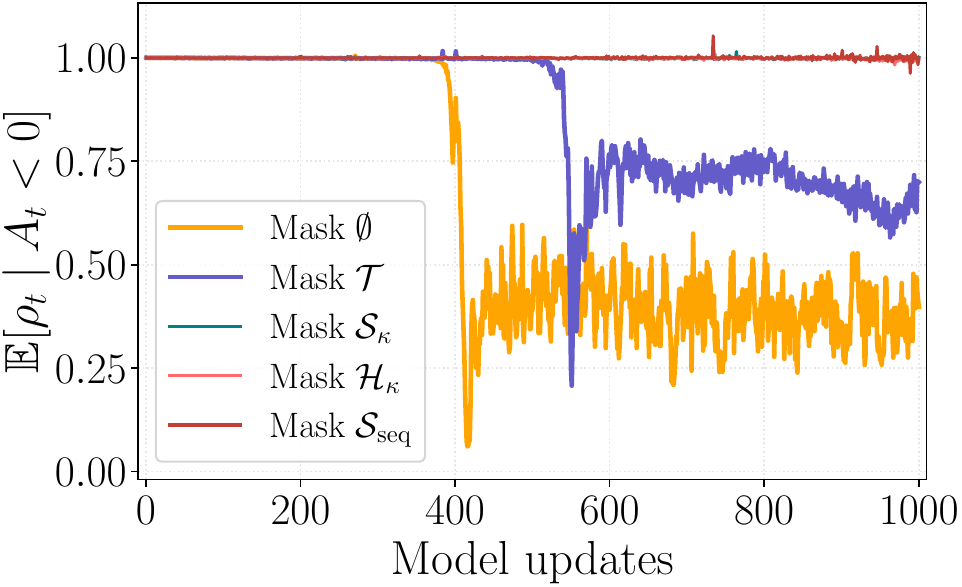}
\caption{\textbf{Localizing the harmful updates.} Under relaxed clipping $[0,5]$, masking only trigger tokens $\mathcal T$ still collapses, while masking the non-trigger suffix $\mathcal H_\kappa$ or broader scopes keeps both accuracy and negative-advantage importance ratios stable. Results are on Qwen2.5-Math-7B.}
\label{fig:nav_explanation}
\end{figure*}

\subsection{Post-Trigger Suffix Token as the Source of Collapse}
\label{subsec:localizing_harmful_updates}

Now that low-ratio negative-advantage tokens are an early warning signal for collapse, we ask what they are warning us about. Are these extremely low-ratio tokens themselves the harmful updates, or do they mark a boundary after which the rest of the response becomes unsafe to learn from?

We define this boundary. A token is a \textit{trigger token} if it has negative advantage and its importance ratio falls below a small threshold: $A_t<0\text{ and }\rho_t(\theta)<\tau_c ,$ where $\tau_c$ is set to a very small value, such as $10^{-4}$. For a response $y$, the \textit{trigger boundary} is the first trigger-token position
\begin{equation}
    \kappa(y)=\min\{t: A_t<0,\ \rho_t(\theta)<\tau_c\},
\end{equation} 
with $\kappa(y)=\infty$ if no such token exists.

With this defined, we compare five masking scopes, all applied only to negative-advantage responses:
\begin{itemize}
[leftmargin=*,noitemsep,topsep=0pt]
\vspace{-5pt}
    \item No-Mask Baseline: Keeps the full, unaltered response.
    \item Trigger-token Mask $\mathcal{T}(y)$: Removes all trigger tokens.
    \item Suffix Mask $\mathcal{S}_{\kappa}(y)$: Removes all tokens from the trigger boundary (exclusive) to the end.
    \item Non-trigger Suffix Mask $\mathcal{H}_{\kappa}(y)= \mathcal S_{\kappa}(y) \setminus \mathcal T(y)$: Masks the suffix tokens excluding trigger tokens.
    \item Sequence Mask $\mathcal{S}_{\mathrm{seq}}(y)$: Masks the whole response once any trigger appears.
\end{itemize}

Figure~\ref{fig:nav_explanation} shows the results. Under relaxed clipping $[0,5]$, both no-mask baseline and trigger-token mask $\mathcal T(y)$ collapse. Thus, low-ratio tokens are primarily  \textit{triggers}, rather than \textit{the dominant harmful updates}. By contrast, non-trigger suffix mask $\mathcal H_{\kappa}(y)$ stabilizes training, and the broader masks $\mathcal S_{\kappa}(y)$ and $\mathcal S_{\mathrm{seq}}(y)$ are also stable because they include the same post-boundary suffix.

The takeaway is that \textit{collapse is driven by post-trigger suffix tokens with non-tiny local ratios, not by the trigger tokens alone}. Once a response hits a low-ratio negative-advantage token, later tokens are evaluated under a prefix that the current policy has already made unlikely, so their updates can become unsafe even when their own ratios are not extremely small. This explains why token-level clipping fails under high staleness: it suppresses updates based only on local ratios, leaving post-boundary suffix tokens active whenever their own ratios are not extremely small.

\paragraph{Why Post-Trigger Tokens Become Harmful.}
\label{paragraph:post_boundary_mechanism}
The ablation above suggests that the low-ratio token is not the dominant source of collapse. Instead, it marks the point after which tokens are conditioned on prefixes that the current policy $\pi_\theta$ is unlikely to generate. Define the prefix occupancy ratio as
\begin{equation}
    R_{t-1}(\theta)
    =
    \prod_{j=1}^{t-1}\rho_j(\theta)
    =
    \frac{\pi_\theta(a_{<t}\mid x)}
         {\beta(a_{<t}\mid x)} .
    \label{eq:prefix_ratio}
\end{equation}
Once a trigger appears at position $\kappa$, its tiny ratio enters every later $R_{t-1}$, so teacher-forcing training continues to update suffix tokens behind a prefix with low current-policy support.

This is the correction that token-level GRPO omits. It weights each suffix token by only its local ratio $\rho_t(\theta)$, without accounting for the probability of reaching the prefix on which that token is conditioned. Therefore, post-boundary suffix tokens can still receive large updates when their own ratios are non-tiny, even though the prefixes needed to reach them are already unlikely.

This explains why $\mathcal H_\kappa(y)$, rather than $\mathcal T(y)$, is the critical harmful set. Trigger tokens in $\mathcal T(y)$ have tiny local ratios, so their direct negative-gradient coefficients are already small. Tokens in $\mathcal H_\kappa(y)$, however, inherit low prefix support while retaining non-tiny local ratios, so token-level GRPO keeps their gradients active. As a result, masking $\mathcal T(y)$ still collapses, whereas masking $\mathcal H_\kappa(y)$ prevents collapse. Appendix~\ref{app:prefix-divergence} formalizes this view. 
\section[Method: mu-GRPO]{Method: \texorpdfstring{\boldmath$\mu$-GRPO}{mu-GRPO}}
\label{sec:method}

Section~\ref{sec:diagnosis} shows that high-staleness rollouts can retain useful learning signal once their localized instability is controlled. This staleness tolerance allows us to safely pursue high-staleness training, where a large static rollout dataset supports many optimizer updates. As a result, we can reduce frequent rollout refreshes and translate algorithmic tolerance into system efficiency.

\paragraph{\boldmath \ourmethod{}.} We propose a staged training framework that separates rollout generation and policy optimization into large-$\mu$ phases. Instead of tightly interleaving generation and optimization as in standard GRPO, \ourmethod{} generates rollouts in bulk and then performs many policy updates on the resulting static dataset. Each stage consists of two phases:

\paragraph{Phase 1: High-Throughput Rollout Generation.}
\label{paragraph:phase1}
At the beginning of each stage $k$, we freeze the current policy as the rollout policy $\beta_k$ and use it only for rollout generation. Specifically, for the first stage, we initialize the rollout policy from the base model, $\beta_0 \leftarrow \pi_{\text{base}}$; for later stages, $\beta_k$ is obtained by freezing the policy after the optimization phase of stage $k-1$. Given the stage prompt set $\mathcal{X}_k=\{x_i\}_{i=1}^{B_{\mathrm{train}}}$, the rollout policy $\beta_k$ samples a response group $\mathcal Y_i=\{y_{i,g}\}_{g=1}^G$ for each prompt $x_i$. For each generated response $y=(a_1,\ldots,a_T)$ to prompt $x$, we compute its group-normalized advantage $A$ as in Eq.~\ref{eq:grpo-adv} and store tokenwise behavior log-probabilities $b_{1:T}$ with $b_t=\log\beta_k(a_t\mid x,a_{<t})$. The resulting static rollout dataset is
\[
\mathcal D_k=\{(x_i, y, A, b_{1:T}) : x_i\in\mathcal X_k,\; y\in\mathcal Y_i\}.
\]
The frozen rollout engine emits behavior log-probabilities during decoding, avoiding an additional behavior-policy forward pass. Once rollout generation is complete, the rollout engine can be released from memory, and $\mathcal D_k$ is passed to the optimization phase.

\paragraph{Phase 2: Policy Optimization with Large \boldmath $\mu$.}
\label{paragraph:phase2}

The policy is then optimized on the static rollout dataset $\mathcal D_k$ for $\mu=B_{\mathrm{train}}/B_{\mathrm{mini}}$ updates before the next rollout stage. By making $B_{\mathrm{train}}$ large, \ourmethod{} operates with $\mu$ on the order of thousands, reusing the same behavior policy over many updates. During optimization, $\beta_k$ is fixed and its behavior log-probabilities are already stored, so the rollout engine need not remain resident. For a token $a_t$ in response $y$ to prompt $x$, with stored behavior log-probability $b_t$, we compute the importance ratio as
\[
    \rho_t(\theta)
    =
    \exp\left(
        \log \pi_\theta(a_t\mid x,a_{<t}) - b_t
    \right).
\]
Motivated by the analysis in \cref{sec:diagnosis}, \ourmethod{} applies \textit{negative-advantage veto} before the GRPO update. For each negative-advantage response $y$, the veto rule checks whether any token satisfies $\rho_t(\theta)<\tau_c$. We consider three veto scopes that contain the diagnosed non-trigger suffix $\mathcal H_\kappa$: $\mathcal H_\kappa$, $\mathcal S_\kappa$, and $\mathcal S_{\mathrm{seq}}$. We use $\mathcal S_{\mathrm{seq}}$ by default for simplicity and robustness, and ablate the scope choice in Appendix~\ref{app:nav_scope_ablation}. Under this default, once a trigger is detected, the entire response is masked. Thus, in the objective below, $m_{t}(\theta)$ is zero
for all tokens in a triggered negative-advantage response and equals one otherwise. The mask is recomputed at each update and is not differentiated through. 

With this mask, we optimize a relaxed clipped GRPO objective with clipping range $[0,5]$:
\[
J_{\mu\text{-GRPO}}(\theta)
=
\mathbb{E}_{\mathcal B\sim\mathcal D_k}
\left[
\frac{1}{|\mathcal B|}
\sum_{(x,y,A,b_{1:T})\in\mathcal B}
\frac{1}{T}
\sum_{t=1}^{T}
m_t(\theta)
\min\left(
\rho_t(\theta)A,
\operatorname{clip}(\rho_t(\theta),0,5)A
\right)
\right].
\]

\begin{figure*}[!h]
\vspace{-1.2em}
    \centering
    \includegraphics[width=0.95\textwidth]{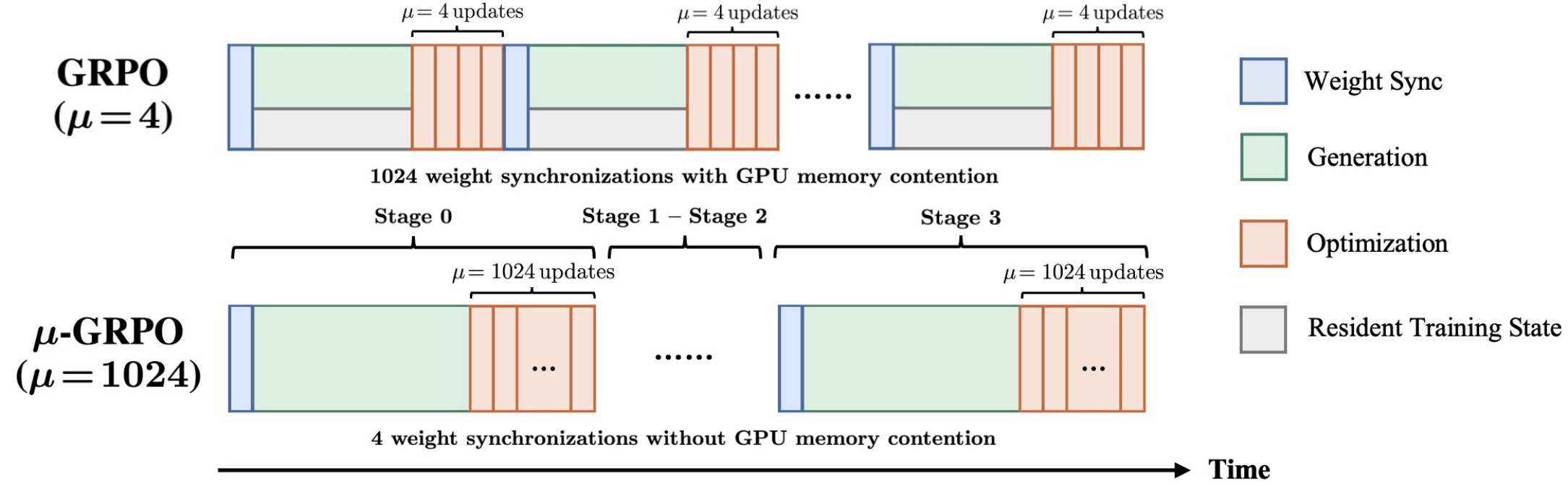}
    \caption{
\textbf{System-level execution of standard GRPO and \boldmath\ourmethod.} Over 4096 model updates, standard GRPO with $\mu\!=\!4$ requires 1024 rollout refreshes and weight synchronizations, while \ourmethod{} with $\mu\!=\!1024$ requires only four. Large-stage rollout generation reduces synchronization overhead and avoids GPU memory contention during generation.
    }
    \label{fig:framework}
\vspace{-0.9em}
\end{figure*} 
\paragraph{Multi-Stage \boldmath $\mu$-GRPO.}
\label{sec:multistage}
\setlength{\columnsep}{10pt}
\begin{wrapfigure}{r}{0.45\linewidth}
\vspace{-1.35em}
\centering
\includegraphics[width=\linewidth]{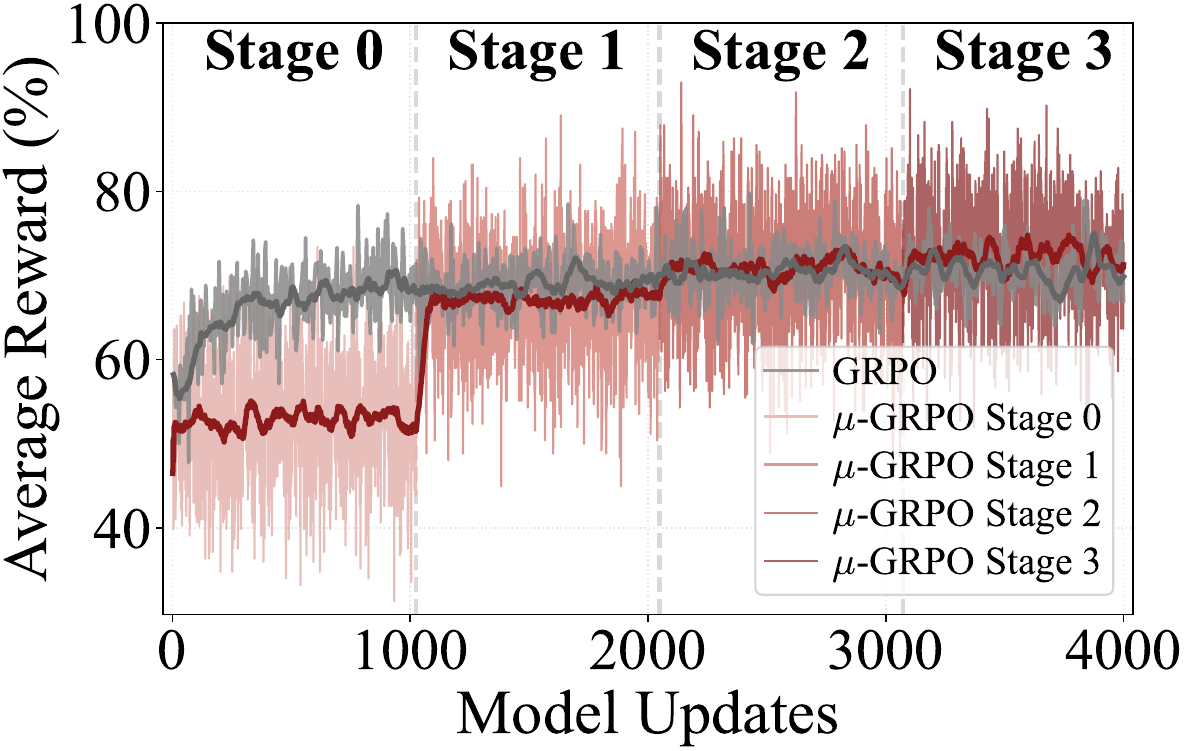}
\vspace{-1.5em}
\caption{
\textbf{Training-set reward dynamics} on DeepSeek-7B. Bold lines are smoothed.
}
\label{fig:reward}
\vspace{-12pt}
\end{wrapfigure}
Both standard GRPO and $\mu$-GRPO can be viewed as repeated generation--optimization cycles, as illustrated in \cref{fig:framework}. The key difference is the scale of each cycle: standard GRPO refreshes the behavior policy after small rollout batches, keeping $\mu$ low, whereas $\mu$-GRPO generates a large static rollout dataset and performs many updates before the next rollout stage. The two-phase procedure above defines one such large-$\mu$ stage. After stage $k$ consumes $\mathcal{D}_k$, the updated policy is frozen as the rollout policy for stage $k+1$, and a fresh dataset $\mathcal{D}_{k+1}$ is generated. This design separates two timescales: within each stage, the rollout policy stays fixed for many updates to amortize generation; across stages, it is refreshed only a few times to improve rollout quality. As shown in \cref{fig:reward,fig:teaser}, this sparse refresh schedule improves training-set reward across stages and translates the large-$\mu$ efficiency gain into stable evaluation performance.

\paragraph{System-Level Efficiency.}
\label{sec:method_efficiency} 
Figure~\ref{fig:framework} compares the execution schedules of standard GRPO and \ourmethod{} over 4096 model updates. Standard GRPO with $\mu=4$ requires 1024 rollout refreshes and weight synchronizations between a high-throughput inference engine, such as vLLM~\citep{kwon2023efficient} or SGLang~\citep{zheng2024sglang}, and a sharded training engine, such as FSDP~\citep{zhao2023pytorchfsdpexperiencesscaling} or Megatron-LM~\citep{DBLP:journals/corr/abs-1909-08053}. Each refresh can incur sharded-parameter materialization, transfer to the rollout engine, and inference-layout conversion. Frequent alternation also creates GPU memory contention: to switch quickly between generation and optimization, systems often keep model parameters and optimizer states resident during rollout. This reduces memory available for high-throughput inference.

\ourmethod{} reduces both costs by operating at large rollout staleness ($\mu\!=\!1024$): Over 4096 model updates, it performs only four rollout stages and weight synchronizations. Within each stage, rollout generation is executed as a standalone phase: the frozen rollout policy generates a large static dataset, after which the rollout engine can be released and the trainer consumes the data using stored behavior log-probability anchors. With only a few generation--optimization switches, training state need not remain co-resident during rollout, leaving more device memory for larger rollout batches and higher decoding parallelism. Large-$\mu$ training therefore turns GRPO's staleness tolerance into lower synchronization overhead, lower memory contention, and higher rollout throughput.
\section{Experiments}
\label{sec:experiment}

\paragraph{Models \& Datasets.} We evaluate \ourmethod\ on five language models: Llama-3.2-3B-Instruct~\citep{grattafiori2024llama3herdmodels}, Qwen2.5-Math-7B~\citep{yang2024qwen25mathtechnicalreportmathematical}, Qwen3-Base-1.7B and 8B~\citep{yang2025qwen3technicalreport}, and DeepSeek-7B~\citep{shao2024deepseekmathpushinglimitsmathematical}. We use a context length of 4k for Qwen2.5-Math-7B, which is the maximum supported by this model, and 8k for all other models. For training, we adopt the DeepScaleR~\citep{sun2026improving} mathematics dataset.

\paragraph{Training \& Evaluation.} We implement all methods in verl~\citep{sheng2025hybridflow}, with  vLLM~\citep{kwon2023efficient} for rollout generation and FSDP~\citep{zhao2023pytorchfsdpexperiencesscaling} for actor optimization. Each configuration is trained for 4096 model update steps with batch size 32, giving the same rollout-sample budget. We evaluate GRPO under both low- and high-staleness settings ($\mu=4$ and $\mu=1024$), and set $\mu=1024$ for \ourmethod{} and M2PO. High-staleness GRPO uses the same two-phase large-stage rollout pipeline as \ourmethod{} but optimizes the GRPO objective with default clipping range and no negative-advantage veto. Thus, \ourmethod{} performs four rollout generation stages, each followed by 1024 updates, as illustrated in \cref{fig:teaser}. For GRPO, we set $\epsilon=0.2$ and omit the KL loss following prior work~\citep{yu2026dapo,zheng2026prosperity}. For M2PO, we use the recommended second-moment threshold of $0.04$.

We use a fully aligned wall-clock measurement protocol on identical hardware configurations with $4\times$H200 GPUs, including both rollout generation and policy optimization. For evaluation, we use five widely adopted math reasoning benchmarks: AMC23/24~\citep{aops_amc}, AIME24/25~\citep{aops_aime}, and Math500~\citep{hendrycks2021measuringmathematicalproblemsolving,lightman2024let}. Models are evaluated every 200 model updates, and we report the checkpoint with the best average performance across all benchmarks. Additional experimental details are provided in Appendix~\ref{app:details}.

\begin{table*}[tbh]
\tabcolsep 2.3pt
\caption{{Performance (\%) and efficiency comparison} across five math reasoning benchmarks and five models. \textbf{Bold} marks the best Avg.; \underline{underline} marks the fastest $T_{\mathrm{total}}$ among methods no worse than standard GRPO in Avg. Overall, $\mu$-GRPO achieves the strongest performance--efficiency trade-off.}
\label{tab:performance_comparison}
\centering
\small
\begin{sc}
\begin{tabular}{c|c|ccccc|c|cc}
\toprule
Methods & $\mu$
& AMC23 & AMC24
& AIME24 & AIME25
& Math500
& Avg. $\uparrow$
& $T_\mathrm{rollout}$\,\textnormal{(h)} $\downarrow$
& $T_\mathrm{total}$\,\textnormal{(h)} $\downarrow$ \\
\midrule\midrule

\multicolumn{10}{c}{Llama-3.2-3B-Instruct} \\
\midrule
GRPO & 4
& 25.0 & 20.6 & 12.3 & 0.2 & 55.1 & 22.6 & 15.1 & 22.7 \\
GRPO & 1024
& 21.9 & 12.2 & 9.8 & 0.6 & 46.8 & 18.3 & 13.2 & 19.3 \\
M2PO & 1024
& 30.6 & 15.0 & 10.6 & 0.2 & 50.7 & 21.4 & 15.3 & 21.3 \\
\rowcolor{gray!20}
\ourmethod & 1024
& 29.4 & 17.7 & 14.0 & 0.5 & 52.8 & \textbf{22.9} & 12.7 & \underline{{18.6}} \\
\midrule

\multicolumn{10}{c}{Qwen2.5-Math-7B} \\
\midrule
GRPO & 4
& 68.1 & 42.8 & 29.0 & 15.8 & 81.0 & 47.3 & 8.8 & 23.7 \\
GRPO & 1024
& 61.3 & 36.2 & 23.3 & 9.8 & 78.2 & 41.8 & 3.8 & 17.9 \\
M2PO & 1024
& 66.9 & 40.6 & 29.6 & 14.6 & 82.5 & 46.8 & 6.7 & 20.2 \\
\rowcolor{gray!20}
\ourmethod & 1024
& 63.1 & 47.2 & 31.5 & 16.5 & 81.7 & \textbf{48.0} & 4.7 & \underline{{18.1}} \\
\midrule

\multicolumn{10}{c}{Qwen3-Base-1.7B} \\
\midrule
GRPO & 4
& 42.5 & 27.2 & 7.5 & 3.8 & 66.6 & 29.5 & 11.2 & 20.5 \\
GRPO & 1024
& 32.5 & 19.3 & 7.7 & 4.5 & 63.5 & 25.5 & 6.5 & 13.3 \\
M2PO & 1024
& 35.0 & 24.4 & 10.0 & 4.4 & 66.6 & 28.1 & 12.9 & 24.2 \\
\rowcolor{gray!20}
\ourmethod & 1024
& 40.8 & 24.9 & 7.5 & 6.5 & 68.7 & \textbf{29.7} & 5.2 & \underline{{12.4}} \\
\midrule

\multicolumn{10}{c}{Qwen3-Base-8B} \\
\midrule
GRPO & 4
& 68.8 & 53.3 & 23.5 & 19.4 & 86.3 & \textbf{50.3} & 21.3 & \underline{{58.2}} \\
GRPO & 1024
& 59.6 & 43.2 & 17.6 & 11.2 & 74.6 & 41.2 & 7.4 & 23.6 \\
M2PO & 1024
& 65.6 & 53.9 & 24.2 & 19.6 & 87.1 & 50.1 & 23.2 & 49.7 \\
\rowcolor{gray!20}
\ourmethod & 1024
& 69.4 & 51.1 & 21.5 & 17.4 & 81.8 & 48.2 & 7.9 & {25.9} \\
\midrule

\multicolumn{10}{c}{DeepSeek-7B} \\
\midrule
GRPO & 4
& 91.4 & 66.7 & 46.8 & 33.2 & 92.4 & 66.1 & 26.0 & 61.9 \\
GRPO & 1024
& 82.6 & 63.5 & 39.5 & 26.8 & 88.7 & 60.2 & {22.3} & 54.5 \\
M2PO & 1024
& 90.6 & 69.4 & 50.8 & 35.0 & 93.0 & 67.8 & {27.9} & 68.3 \\
\rowcolor{gray!20}
\ourmethod & 1024
& 93.1 & 68.9 & 50.2 & 37.9 & 92.8 & \textbf{68.6} & 21.4 & \underline{{51.3}} \\
\bottomrule
\end{tabular}
\end{sc}
\end{table*}
\subsection{Main Results}
\label{sec:analysis}
We analyze the results in \cref{tab:performance_comparison} along two dimensions: final reasoning performance and wall-clock efficiency, measured by rollout-generation time $T_{\mathrm{rollout}}$ and end-to-end training time $T_{\mathrm{total}}$.

\textbf{\boldmath $\mu$-GRPO \textit{vs.} High-Staleness GRPO: Recovering performance under the same large-$\mu$ regime.} Naively increasing GRPO to $\mu\!=\!1024$ reduces rollout-generation and wall-clock time, but degrades accuracy. Under the same high-staleness setting, $\mu$-GRPO recovers this performance loss while retaining the efficiency benefits of large-stage rollout execution. It improves average accuracy from 41.8\% to 48.0\% on Qwen2.5-Math-7B and from 60.2\% to 68.6\% on DeepSeek-7B. The gains are visible on harder benchmarks such as AIME24/25. This consistent gap shows that the stabilization mechanisms in $\mu$-GRPO are critical for turning high-staleness data reuse into effective learning. 

\textbf{\boldmath \ourmethod{} vs. Standard GRPO and M2PO: Better performance--efficiency trade-off.} Against standard GRPO ($\mu\!=\!4$), \ourmethod{} matches or improves average accuracy on four out of five models, with the largest gain on DeepSeek-7B (68.6\% vs.\ 66.1\%), while reducing $T_{\mathrm{rollout}}$ by {1.82$\times$} and $T_{\mathrm{total}}$ by {1.53$\times$} on average. Compared to M2PO at $\mu\!=\!1024$, \ourmethod{} achieves competitive accuracy while reducing $T_{\mathrm{rollout}}$ by {1.87$\times$} and $T_{\mathrm{total}}$ by {1.49$\times$}. M2PO confirms that high-staleness rollouts can remain learnable, while \ourmethod{} more directly converts this tolerance into rollout and wall-clock efficiency gains. Overall, \ourmethod{} achieves a stronger performance--efficiency trade-off.

\paragraph{Evaluation Results on Coding Tasks.} As a test beyond math reasoning, we train Qwen2.5-Coder-7B~\citep{hui2024qwen25codertechnicalreport} on the \texttt{code\_contests} dataset~\citep{li2022competition} and evaluate on LiveCodeBench~\citep{jain2024livecodebenchholisticcontaminationfree}, using the same setup as in the main experiments, except for a minibatch size of 16. Starting from 8.2\% accuracy, GRPO and \ourmethod{} reach comparable accuracies of 39.3\% and 38.6\%. However, \ourmethod{} reduces wall-clock training time from 98.9h to 45.7h, yielding a 2.16$\times$ speedup. This suggests that the efficiency benefit of large-$\mu$ training can transfer to coding tasks without sacrificing performance.

\paragraph{Asynchronous Execution.}
\begin{wraptable}{r}{0.56\textwidth}
\vspace{-1.25em}
\centering
\setlength{\tabcolsep}{0.2cm}
\renewcommand{\arraystretch}{0.8}
\footnotesize
\caption{Asynchronous execution on Qwen2.5-Math-7B.}
\vspace{-0.6em}
\label{tab:async}
\begin{tabular}{ccccc}
\toprule
Method & Max Lag & Avg.$\uparrow$ & $T_{\mathrm{total}}$\,(h)$\downarrow$ & Idle Ratio$\,(\%)\downarrow$ \\
\midrule
GRPO & 4 & 47.9 & 22.3 & 24.2 \\
$\mu$-GRPO & 1024 & 47.7 & \textbf{17.4} & \textbf{0.6} \\
\bottomrule
\end{tabular}
\vspace{-1.2em}
\end{wraptable}
Unlike the synchronous staged setting in \cref{fig:framework}, we evaluate GRPO and $\mu$-GRPO in a fully asynchronous setup with rollout workers and trainers running in parallel on separate devices. We use $4\times$H200 GPUs, split evenly between rollout workers and trainers. Rollout samples are streamed through a queue, and workers periodically receive synchronized policy parameters. With periodic synchronization, samples may lag behind the trainer policy; we cap the maximum lag for async $\mu$-GRPO to match the $\mu=1024$ staleness level of the synchronous setting. Async $\mu$-GRPO uses the same infrastructure as async GRPO, but replaces the trainer-side update with relaxed clipping and negative-advantage veto. This tests whether $\mu$-GRPO's staleness tolerance transfers from serial staged training to parallel asynchronous execution. As shown in \cref{tab:async}, async $\mu$-GRPO maintains comparable accuracy while reducing training time from 22.3h to 17.4h; the trainer idle ratio, defined as the fraction of total trainer step time spent waiting for rollout samples, drops from 24.2\% to 0.6\%. 
\subsection{Ablation Study}
\begin{table}[!h]
\centering
\setlength{\tabcolsep}{0.08cm}
\small
\begin{minipage}[t]{.48\linewidth}
\centering
\caption{\textbf{Fixed generation budget.} Varying batch size on a fixed rollout dataset.}
\label{tab:batch-mu-ablation}
\resizebox{\linewidth}{!}{
\begin{tabular}{cccccccc}
\toprule
\scriptsize Batch & \scriptsize $\mu$
& \scriptsize AMC23 & \scriptsize AMC24
& \scriptsize AIME24 & \scriptsize AIME25
& \scriptsize MATH500
& \scriptsize Avg. \\
\midrule
\rowcolor{gray!20} 
32  & 1259 & 61.9 & 40.6 & 24.8 & 13.8 & 80.0 & \textbf{44.2} \\
64  & 629  & 63.1 & 39.4 & 23.3 & 13.3 & 79.7 & 43.8 \\
128 & 314  & 61.3 & 36.1 & 21.5 & 13.5 & 80.7 & 42.6 \\
256 & 157  & 62.5 & 32.8 & 23.1 & 11.9 & 79.3 & 41.9 \\
\bottomrule
\end{tabular}}
\end{minipage}
\hfill
\begin{minipage}[t]{.486\linewidth}
\centering
\caption{\textbf{Fixed optimization budget.}
Varying $\mu$ with 4096 updates.}
\label{tab:staleness_ablation}
\resizebox{\linewidth}{!}{
\begin{tabular}{cccccccc}
\toprule
\scriptsize $\mu$
& \scriptsize AMC23
& \scriptsize AMC24
& \scriptsize AIME24
& \scriptsize AIME25
& \scriptsize MATH500
& \scriptsize Avg.
& \scriptsize \#Stage \\
\midrule
512  & 65.6 & 43.9 & 30.8 & 13.3 & 82.9 & 47.3 & 8 \\
\rowcolor{gray!20}
1024  & 63.1 & 47.2 & 31.5 & 16.5 & 81.7 & \textbf{48.0} & 4 \\
2048 & 65.7 & 40.6 & 28.1 & 16.9 & 79.1 & 46.1 & 2 \\
4096 & 66.9 & 37.2 & 26.0 & 16.0 & 82.3 & 45.7 & 1 \\
\bottomrule
\end{tabular}}
\end{minipage}
\end{table}
\paragraph{Effect of Staleness with a Fixed Generation Budget.} To isolate the value of reusing stale rollout data, we first generate a fixed rollout dataset by sampling 8 responses for each prompt in the DeepScaleR training set, and then optimize only on this dataset. Under this setting, the batch size determines both the number of updates and the effective staleness $\mu$: smaller batches yield more updates but higher staleness, while larger batches reduce staleness but shorten training. \cref{tab:batch-mu-ablation} varies the batch size in $\{32,64,128,256\}$. Batch size 32 ($\mu=1259$) achieves the best average accuracy of 44.2\%, while batch size 256 drops to 41.9\%. This suggests that rollout data remain useful for many updates, and that additional updates can outweigh the cost of higher staleness.

\paragraph{Effect of Staleness with a Fixed Optimization Budget.} \cref{tab:staleness_ablation} fixes the total number of model updates and varies the effective staleness $\mu$ by changing the rollout refresh frequency, or equivalently the number of generation--optimization stages. Performance remains stable from $\mu\!=\!512$ to $\mu\!=\!2048$, but degrades when using a single stage ($\mu\!=\!4096$). This suggests that stale rollouts remain useful over many updates, but occasional refreshes are still important for maintaining rollout quality. Among all settings, $\mu\!=\!1024$ provides the best average accuracy and is adopted as the default configuration.
\section{Related Work}
\label{sec:relatedworks}
We summarize key related work here and provide a more detailed discussion in Appendix~\ref{app:more_related_work}. Practical RLVR systems are partially off-policy due to rollout reuse ~\citep{shao2024deepseekmathpushinglimitsmathematical, openr1}, rollout--training decoupling in modern RL systems ~\citep{sheng2025hybridflow, slime_github, fu2026areal, noukhovitch2025asynchronous, zhong2025streamrl, he2025history}, and training--inference mismatch~\citep{qi2025defeating}. Prior work stabilizes off-policy RLVR using clipping or trust-region variants, including ratio clipping~\citep{yu2026dapo}, approximate trust regions~\citep{fu2026areal}, sequence-level clipping~\citep{zheng2025group}, asymmetric clipping~\citep{Roux2025TaperedOR}, and gradient-preserving clipping~\citep{su2025klear, chen2025minimax}. The closest work to ours is M2PO~\citep{zheng2026prosperity}, which controls instability through batch-level masking of high-variance tokens. In contrast, we push GRPO to extreme rollout staleness, diagnose its failure mode as localized negative-advantage suffix drift, and stabilize training with a simple trigger--scope filter. Our work complements synchronous/hybrid orchestration systems~\citep{sheng2025hybridflow, slime_github} and asynchronous RLVR systems~\citep{fu2026areal, zhong2025streamrl, han2025asyncflow, noukhovitch2025asynchronous}: we study how to increase GRPO's algorithmic tolerance to rollout staleness, reducing synchronization pressure in both settings. \ourmethod{} realizes this tolerance in a synchronous large-stage framework, where rollout generation and optimization are executed sequentially across stages rather than overlapped in parallel asynchronously.
\section{Conclusion} \label{sec:conclusion_and_discussion}  This work revisits the common assumption that GRPO must remain in a low-staleness, near-on-policy regime to train LLMs effectively. We show that this constraint is overly conservative: stale rollouts can retain useful learning signal far beyond the standard regime, provided that the localized instability of high-staleness optimization is handled explicitly. Motivated by this diagnosis, we propose $\mu$-GRPO, a staged framework that decouples rollout generation from policy optimization and stabilizes training under extreme staleness. Across multiple models and reasoning benchmarks, $\mu$-GRPO converts this staleness tolerance into a substantially improved performance--efficiency trade-off, delivering notable wall-clock speedups while preserving or improving final performance. Extending high-staleness training to broader domains and larger models remains future work.
\clearpage
\paragraph{Acknowledgements.}
We gratefully acknowledge the Center for Research Computing (CRC) at Rice University for providing computational resources, technical support, and research computing services that supported this work. We also thank Jefferson Hernandez for the helpful discussions and constructive suggestions

\bibliographystyle{plainnat} 
\bibliography{references} 
\appendix
\onecolumn
\section{Overview}
\label{app:overview}
This appendix provides details and analyses that support the main text. Appendix~\ref{app:details} describes the experimental setup, including model and dataset details, training hyperparameters, efficiency measurement, evaluation protocols, and the coding and asynchronous experiment settings. Appendix~\ref{app:nav_relaxed_clipping} provides additional analyses of negative-advantage veto and relaxed clipping, including veto-rate statistics, negative-advantage veto scope ablations, stabilization hyperparameter sensitivity, and low-staleness \ourmethod{}. Appendix~\ref{app:prefix-divergence} formalizes the prefix-support mismatch mechanism behind retained harmful suffixes. Appendix~\ref{app:more_related_work} discusses additional related work.
\section{Detailed Experimental Settings} 
\label{app:details}  
\subsection{Main Experiment Settings}
\paragraph{Models \& Dataset.} All experiments are conducted on five backbone language models: Llama-3.2-3B-Instruct~\citep{grattafiori2024llama3herdmodels}, Qwen2.5-Math-7B~\citep{yang2024qwen25mathtechnicalreportmathematical}, Qwen3-Base-1.7B and 8B~\citep{yang2025qwen3technicalreport}, and DeepSeek-7B~\citep{shao2024deepseekmathpushinglimitsmathematical}. DeepSeek-7B denotes the Qwen2.5-Math-7B model distilled from DeepSeek-R1. Following model specifications, we use a context length of 4k tokens for Qwen2.5-Math-7B, which is the maximum supported by this model, and 8k tokens for all other models. All models are trained on the DeepScaleR mathematics dataset~\citep{sun2026improving}. For each training instance, we use a unified instruction-style prompt template. The model is instructed to reason step by step and return the final answer enclosed in \verb|\boxed{}| tags, which enables reliable answer extraction during training and evaluation.  
\begin{tcolorbox}[ colback=gray!5, colframe=black, coltitle=white, title=\textbf{Prompt Template}, fonttitle=\bfseries, arc=6pt, boxrule=1pt, left=8pt, right=8pt, top=6pt, bottom=6pt ] Please solve the following math problem: \{\{Question Description\}\}. The assistant first thinks about the reasoning process step by step and then provides the user with the answer. Return the final answer in \verb|\boxed{}| tags, for example \verb|\boxed{1}|. Let's solve this step by step. 
\end{tcolorbox}  
\vspace{-0.6em}
\paragraph{Training.} All methods are implemented in verl~\citep{sheng2025hybridflow}, with vLLM~\citep{kwon2023efficient} for rollout generation and FSDP~\citep{zhao2023pytorchfsdpexperiencesscaling} for actor optimization. Policy optimization is performed on the actor network only. We use AdamW with a constant learning rate of $1\times10^{-6}$, $\beta_1=0.9$, $\beta_2=0.999$, and weight decay 0.01. Following prior work~\citep{yu2026dapo, zheng2026prosperity}, we omit the explicit KL-divergence regularization term. Rewards are binary: 1.0 if the extracted answer is correct, and 0.0 otherwise.  For rollout generation, we use temperature 1.0 and sample 8 responses per prompt. Each configuration is trained for 4096 model updates with batch size 32, giving the same rollout-sample budget across methods. We evaluate GRPO under both low- and high-staleness settings, $\mu=4$ and $\mu=1024$, and set $\mu=1024$ for \ourmethod{} and M2PO. The high-staleness GRPO baseline uses the same two-phase large-stage rollout pipeline as \ourmethod{}, but optimizes the standard GRPO objective with default clipping and no negative-advantage veto. This isolates the effect of the proposed stabilization mechanism from staged rollout execution. Thus, \ourmethod{} uses four rollout stages, each followed by 1024 model updates.  For GRPO, we set $\epsilon=0.2$. For \ourmethod{}, we use the sequence-level negative-advantage veto scope $\mathcal S_{\mathrm{seq}}$, a conservative superset of the harmful suffix $\mathcal H_\kappa$ identified in Section~\ref{sec:diagnosis}, and set $\tau_c=10^{-4}$; Appendix~\ref{app:nav_scope_ablation} ablates alternative scopes. For M2PO, we use the recommended second-moment threshold of 0.04.  

\paragraph{Efficiency Measurement.} All efficiency experiments use identical $4\times$H200 GPU configurations without parameter offloading. We report two timing metrics. $T_{\mathrm{rollout}}$ denotes the cumulative wall-clock time of all rollout-generation phases across training. $T_{\mathrm{total}}$ denotes the end-to-end wall-clock training time, including rollout generation, policy optimization, synchronization, data movement, and scheduling overhead. This aligned timing protocol is used for all methods in \cref{tab:performance_comparison}. 

\paragraph{Evaluation.} We evaluate trained models on AMC23/24~\citep{aops_amc}, AIME24/25~\citep{aops_aime}, and Math500~\citep{hendrycks2021measuringmathematicalproblemsolving, lightman2024let}. Models are evaluated every 200 model updates, and we report the checkpoint with the best average performance across all benchmarks. Evaluation uses pass@1 accuracy. For AIME24/25, pass@1 is averaged over 16 sampled generations; for AMC23/24 and Math500, pass@1 is averaged over 4 sampled generations. All evaluations use the same context length as training. 

\subsection{Additional Experiment Settings}
\paragraph{Coding Experiment Details.} For the coding experiment, we train Qwen2.5-Coder-7B~\citep{hui2024qwen25codertechnicalreport} on the \texttt{code\_contests} dataset~\citep{sun2026improving} and evaluate on LiveCodeBench~\citep{jain2024livecodebenchholisticcontaminationfree}. We use an 8k context length and report pass@4 accuracy. The training setup follows the main experiments except that the per-update batch size is 16. Both GRPO and \ourmethod{} are trained for 4096 updates; \ourmethod{} uses four rollout stages with $\mu=1024$. Wall-clock time is measured with the aligned protocol on $4\times$H200 GPUs.

\paragraph{Asynchronous Experiment Details.}
For the asynchronous experiment, we use the fully asynchronous policy trainer in verl, where rollout workers and trainers are placed on separate devices and run concurrently. We use $4\times$H200 GPUs, with two allocated to rollout workers and two to trainers. Rollout samples are streamed through a message queue, and policy parameters are periodically synchronized from the trainer to the rollout workers. We synchronize parameters every 4 trainer updates for both async runs. Async GRPO uses no additional rollout staleness allowance, while async $\mu$-GRPO allows up to 255 synchronization intervals of stale rollouts. This gives a maximum rollout lag of $(1+255)\times 4=1024$ updates, matching the $\mu=1024$ synchronous setting. Trainer idle ratio is measured as the fraction of trainer step time spent waiting for and assembling rollout samples.
\section{Additional Analysis of Negative-Advantage Veto and Relaxed Clipping}
\label{app:nav_relaxed_clipping}

This section provides additional analyses and ablations for the two stabilization
components of \ourmethod{}: negative-advantage veto (NAV) and relaxed clipping.
We first analyze how often NAV vetoes token-level updates during training, then
compare different NAV masking scopes, study the sensitivity to the NAV threshold
and clipping upper bound, and finally test whether the \ourmethod{} update rule
remains compatible with the standard low-staleness regime.

\subsection{Fraction of Tokens Vetoed by NAV}
\label{app:veto_fraction}

To better understand how much training signal is filtered by NAV, we report the fraction of tokens vetoed during \ourmethod{} training. Recall that under the sequence-level NAV rule, for a negative-advantage response, once a token's importance ratio falls below the threshold $\tau_c$, NAV masks all remaining tokens in the response after the trigger. We define the NAV-vetoed token fraction as the number of tokens masked by NAV divided by the total number of tokens in the optimization batch. Since NAV masks are recomputed online during optimization, this fraction can vary across updates within a stage. These statistics correspond to the main \ourmethod{} runs reported in Table~\ref{tab:performance_comparison}.

\begin{table}[htbp]
\vspace{-8pt}
\centering
\caption{Mean NAV-vetoed token fraction in each stage, computed as vetoed tokens
divided by all tokens in the optimization batch. Avg. is the unweighted average
across stages.}
\label{tab:veto_fraction}
\begin{tabular}{lccccc}
\toprule
Model & Stage 0 & Stage 1 & Stage 2 & Stage 3 & Veto Avg. \\
\midrule
Llama-3.2-3B-Instruct & 6.08\% & 3.28\% & 0.78\% & 2.42\% & 3.14\% \\
Qwen2.5-Math-7B & 15.24\% & 0.27\% & 0.27\% & 0.32\% & 4.03\% \\
Qwen3-Base-1.7B & 7.07\% & 0.22\% & 0.23\% & 0.42\% & 1.99\% \\
Qwen3-Base-8B & 3.78\% & 0.81\% & 2.83\% & 4.72\% & 3.04\% \\
DeepSeek-7B & 0.61\% & 0.01\% & 0.01\% & 0.05\% & 0.17\% \\
\bottomrule
\end{tabular}
\vspace{-1pt}
\end{table}

Table~\ref{tab:veto_fraction} reports the mean NAV-vetoed token fraction in each stage for all five models, along with the unweighted average across stages. Overall, the average vetoed fraction remains small, ranging from 0.17\% on DeepSeek-7B to 4.03\% on Qwen2.5-Math-7B. This shows that NAV removes only a limited subset of token-level updates rather than discarding a large portion of the rollout data. The result supports the view that NAV primarily suppresses harmful off-support negative-advantage updates while preserving most of the training signal.  

The vetoed fraction is especially small in later stages. For Qwen2.5-Math-7B, the mean vetoed token fraction is 15.24\% in Stage 0 but drops to below 0.4\% in all subsequent stages. Similarly, Qwen3-Base-1.7B decreases from 7.07\% in Stage 0 to below 0.5\% afterward, while DeepSeek-7B remains below 0.7\% across all stages. These trends suggest that NAV is most active when the model is early in training or when the rollout distribution is more prone to off-support negative-advantage updates, and becomes increasingly sparse once the staged training process stabilizes.

\begin{figure*}[htbp]
\vspace{-5pt}
\centering
\includegraphics[width=0.6\textwidth]{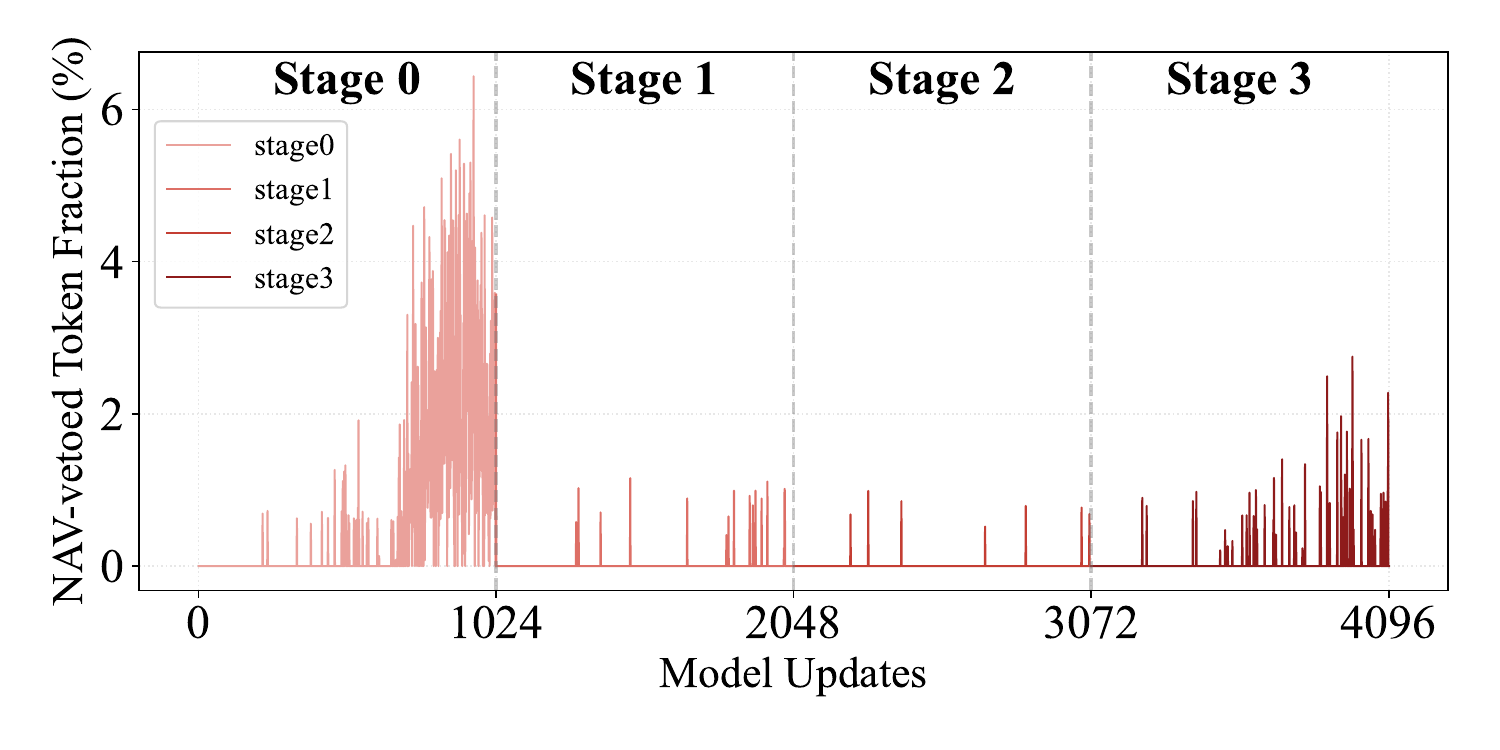}
\vspace{-12pt}
\caption{NAV-vetoed token fraction during $\mu$-GRPO training on DeepSeek-7B.}
\label{fig:veto_frac}
\vspace{-16pt}
\end{figure*}

To complement the averages in Table~\ref{tab:veto_fraction},
Figure~\ref{fig:veto_frac} visualizes the NAV-vetoed token fraction throughout
training for DeepSeek-7B across the four stages. The curve shows that the vetoed
token fraction remains close to zero for most updates, with only occasional peaks
within individual stages. Overall, these results show that NAV stabilizes high-staleness training through sparse token-level intervention.

\subsection{NAV Scope Ablation}
\label{app:nav_scope_ablation}

In the main experiments, we use the sequence-level NAV scope $\mathcal S_{\mathrm{seq}}$ as the default choice. For a negative-advantage response, once a token with importance ratio below $\tau_c$ is detected, it vetoes all tokens in that response. The diagnostic experiments in Section~\ref{sec:diagnosis} identify the post-boundary harmful suffix $\mathcal H_\kappa$ as the critical subset that must be removed to prevent collapse. This suggests that masking scopes that include $\mathcal H_\kappa$, such as $\mathcal S_{\kappa}$ and $\mathcal S_{\mathrm{seq}}$, should also stabilize high-staleness optimization.

To examine the effect of the veto scope, we compare three NAV scopes under the main experimental setting: 4096 model updates with four rollout stages, corresponding to $\mu=1024$ updates per stage. The scopes are the minimal harmful suffix $\mathcal H_\kappa$, the full post-boundary suffix $\mathcal S_{\kappa}$, and the sequence-level scope $\mathcal S_{\mathrm{seq}}$. We report the fraction of vetoed tokens in each stage, computed as the number of vetoed tokens divided by the total number of tokens in the optimization batch. We also report the average veto rate across stages and the average accuracy across the five math benchmarks.

\begin{table}[htbp]
\vspace{-8pt}
\centering
\caption{NAV scope ablation on Qwen2.5-Math-7B. Veto rates are reported per stage.}
\label{tab:nav_scope_ablation}
\begin{tabular}{ccccccc}
\toprule
NAV Scope & Stage 0 & Stage 1 & Stage 2 & Stage 3 & Veto Avg. & Avg. Acc. \\
\midrule
$\mathcal H_\kappa$
& 11.70\% & 0.18\% & 0.21\% & 0.17\% & 3.07\% & 47.9 \\
$\mathcal S_{\kappa}$
& 13.40\% & 0.22\% & 0.23\% & 0.22\% & 3.52\% & {47.6} \\
\rowcolor{gray!20}
$\mathcal S_{\mathrm{seq}}$ 
& 15.24\% & 0.27\% & 0.27\% & 0.32\% & 4.03\% & \textbf{48.0} \\
\bottomrule
\end{tabular}
\vspace{-2pt}
\end{table}

All three scopes remain stable and achieve similar average accuracy. The minimal scope $\mathcal H_\kappa$ vetoes the fewest tokens, whereas $\mathcal S_{\mathrm{seq}}$ vetoes the most because it removes the entire negative-advantage response once a trigger is detected. The comparison suggests that the additional prefix tokens retained by $\mathcal H_\kappa$ and $\mathcal S_{\kappa}$ provide limited benefit once a trigger has appeared. Since the response has negative advantage and has already crossed an off-support boundary under the current policy, the remaining useful contribution from the pre-trigger prefix is small compared with the risk of continuing to update on the same trajectory. Consistent with this view, the more conservative sequence-level scope achieves the best average accuracy despite vetoing more tokens. We therefore use $\mathcal S_{\mathrm{seq}}$ in the main experiments because it is simple, robust, and slightly outperforms the other scopes in this ablation.

\subsection{Stabilization Hyperparameter Sensitivity}
\label{app:stabilization_hyperparams}

We study the sensitivity of \ourmethod{} to two stabilization-related hyperparameters: the NAV threshold $\tau_c$ and the relaxed clipping upper bound. All experiments follow the main experimental setting: Qwen2.5-Math-7B is trained for 4096 model updates with four rollout stages, corresponding to $\mu=1024$ updates per stage. We report the average accuracy across the five math benchmarks.

\paragraph{NAV Threshold Sensitivity.}
The NAV threshold $\tau_c$ determines when a negative-advantage response is treated as crossing the off-support boundary and triggers sequence-level vetoing. Since this threshold is intended to detect extreme low-ratio events, it should be sufficiently small; an overly large threshold may trigger NAV prematurely and remove useful updates.

\begin{table}[htbp]
\vspace{-10pt}
\centering
\small
\caption{NAV threshold sensitivity.}
\label{tab:nav_threshold_ablation}
\begin{tabular}{cc}
\toprule
NAV threshold $\tau_c$ & Avg. Acc. \\
\midrule
$10^{-1}$ & 45.2 \\
$10^{-2}$ & 46.4 \\
\rowcolor{gray!20}
$10^{-4}$ & \textbf{48.0} \\
$10^{-6}$ & 47.6 \\
\bottomrule
\end{tabular}
\vspace{-11pt}
\end{table}

We sweep $\tau_c$ over $\{10^{-1}, 10^{-2}, 10^{-4}, 10^{-6}\}$. As shown in Table~\ref{tab:nav_threshold_ablation}, $\tau_c=10^{-4}$ achieves the best average accuracy of 48.0\%, while $10^{-6}$ remains close at 47.6\%. In contrast, larger thresholds degrade performance, with $10^{-1}$ dropping to 45.2\%. These results suggest that NAV is robust once the threshold is small enough, but setting it too large can over-veto negative-advantage responses before they truly indicate an off-support boundary.

\paragraph{Clipping Upper Bound Sensitivity.}
We also ablate the upper bound of the relaxed clipping interval. In high-staleness training, the standard GRPO clipping range can overly suppress useful gradients from stale but still informative rollouts. \ourmethod{} therefore uses a relaxed clipping range, while relying on NAV to remove destabilizing negative-advantage suffix updates.
\begin{table}[!h]
\vspace{-11pt}
\centering
\small
\caption{Clipping upper bound sensitivity.}
\label{tab:clip_upper_ablation}
\begin{tabular}{cc}
\toprule
Clipping upper bound & Avg. Acc. \\
\midrule
3 & 47.3 \\
\rowcolor{gray!20}
5 & 48.0 \\
10 & 47.8 \\
$+\infty$ & 43.6 \\
\bottomrule
\end{tabular}
\vspace{-3pt}
\end{table}
With the staged pipeline fixed, we vary the clipping upper bound over $\{3, 5, 10, +\infty\}$, where $+\infty$ denotes no upper clipping. Table~\ref{tab:clip_upper_ablation} shows that an upper bound of 5 achieves the best average accuracy, reaching 48.0\%. A tighter bound of 3 slightly underperforms at 47.3\%, suggesting that it still restricts some useful stale-rollout gradients. Increasing the bound to 10 gives comparable  performance at 47.8\%. Removing the upper bound entirely leads to a clear drop to 43.6\%, showing that although relaxed clipping is beneficial, some upper bound control remains necessary for stable optimization. Thus, a moderately relaxed upper bound best balances useful stale-rollout gradients and stability.

\subsection[Low-Staleness μ-GRPO]{Low-Staleness \texorpdfstring{\boldmath$\mu$-GRPO}{μ-GRPO}}
\label{app:onpolicy_relaxed_nav}

We further test whether the \ourmethod{} update rule is compatible with the standard low-staleness regime. Specifically, we run \ourmethod{} with $\mu=4$, matching the rollout refresh frequency of standard GRPO while enabling the relaxed clipping range $[0,5]$ and NAV. This isolates the effect of the update rule from large-staleness rollout reuse. We report the average accuracy across the five math benchmarks and the average token clipping fraction during training.

\begin{table}[tbh]
\vspace{-10pt}
\caption{Low-staleness $\mu$-GRPO on Qwen2.5-Math-7B.}
  \label{tab:onpolicy_relaxed_nav}
  \centering
  \small
\begin{sc}
  \begin{tabular}{c|c|ccccc|cc}
    \toprule
    Methods & $\mu$
    & AMC23 & AMC24
    & AIME24 & AIME25
    & Math500
    & Avg.
    & Clip Frac \\
    \midrule\midrule
    GRPO & 4
      & 68.1 & 42.8 & 29.0 & 15.8 & 81.0 & 47.3 & $4{\times}10^{-4}$ \\
    \ourmethod & 4
      & 69.2 & 42.1 & 30.1 & 14.9 & 82.3 & \textbf{47.7} & $8{\times}10^{-7}$ \\
    \bottomrule
  \end{tabular}
\end{sc}
\end{table}

As shown in Table~\ref{tab:onpolicy_relaxed_nav}, low-staleness \ourmethod{} slightly improves average accuracy over standard GRPO, from 47.3\% to 47.7\%. It also reduces the average clipping fraction from $4\times10^{-4}$ to $8\times10^{-7}$, indicating that the relaxed clipping rule avoids unnecessary clipping when rollout staleness is small. In this low-staleness setting, the NAV veto fraction remains zero throughout training, so the mild gain is primarily attributable to relaxed clipping rather than NAV intervention. These results show that the \ourmethod{} update is compatible with fresh-data learning, and suggest that the standard GRPO clipping range can be overly conservative even when rollouts are relatively fresh.

This result should not be interpreted as a comparison between on-policy and off-policy training. Rather, it shows that the optimization rule itself matters: standard low-staleness GRPO is not necessarily the strongest GRPO-style update. In the main high-staleness setting, \ourmethod{} combines this more permissive update rule with large-stage rollout reuse and NAV-based stabilization, enabling stable and efficient learning under much larger policy drift.

\section{Prefix-Support Mismatch in Retained Harmful Suffixes}
\label{app:prefix-divergence}
This section formalizes the prefix-support mismatch mechanism identified in \cref{sec:diagnosis}. We show why retaining \(\mathcal H_\kappa\)-tokens can be unstable even when their local importance ratios are not tiny: the behavior policy may assign non-negligible mass to these suffix tokens, while the current policy assigns very small mass to the prefixes needed to reach them.

We follow the notation in Section~\ref{sec:preliminaries}: $\beta$ denotes the behavior policy, $\pi_\theta$ denotes the current policy, and $\rho_t(\theta)$ denotes the token-level importance ratio. For notational simplicity, we drop the response index and write a response as $y=(a_1,\ldots,a_T)$.  Assume a finite decoding horizon $T$; variable-length responses can be represented by padding with EOS tokens. Let $\lambda$ be any fixed distribution over token positions $t\in\{1,\ldots,T\}$ with $\lambda_t\ge 0$ and $\sum_t \lambda_t=1$.

For a token occurrence \(z=(x,y,t)\), define the behavior token-occurrence measure
\[
    d\nu_\beta(x,y,t)
    =
    d\mathcal D(x)\,\lambda_t\,\beta(y\mid x),
\]
and the current-prefix occurrence measure
\[
    d\nu_\theta^{\mathrm{pre}}(x,y,t)
    =
    d\mathcal D(x)\,\lambda_t\,
    \pi_\theta(a_{<t}\mid x)\,
    \beta(a_{t:T}\mid x,a_{<t}).
\]
Thus \(\nu_\theta^{\mathrm{pre}}\) changes only the prefix occupancy from \(\beta\) to \(\pi_\theta\), while keeping the replay suffix conditional under \(\beta\). Whenever \(\beta(a_{<t}\mid x)>0\),
\[
    \frac{d\nu_\theta^{\mathrm{pre}}}{d\nu_\beta}(x,y,t)
    =
    R_{t-1}(\theta)
    :=
    \frac{\pi_\theta(a_{<t}\mid x)}{\beta(a_{<t}\mid x)}
    =
    \prod_{j=1}^{t-1}
    \rho_j(\theta),
\]
where
\[
    \rho_t(\theta)
    :=
    \frac{\pi_\theta(a_t\mid x,a_{<t})}
         {\beta(a_t\mid x,a_{<t})}.
\]
Recall from Section~\ref{sec:diagnosis} that the {trigger boundary} is the first trigger-token position
\[
    \kappa(y)
    :=
    \min\{t:\, A(y)<0,\ \rho_t(\theta)<\tau_c\},
\]
with $\kappa(y)=\infty$ if the set is empty. The non-trigger suffix is
\[
    \mathcal H_\kappa(y)
    :=
    \{t>\kappa(y):\, A(y)<0,\ \rho_t(\theta)\ge \tau_c\}.
\]
Let \(\mathcal U_{\mathcal H}(\theta)\subseteq\{(x,y,t): t\in \mathcal H_\kappa(y)\}\) denote any retained subset of these \(\mathcal H_\kappa\)-tokens. If the update rule retains all of \(\mathcal H_\kappa(y)\), then \(\mathcal U_{\mathcal H}(\theta)=\{(x,y,t):t\in \mathcal H_\kappa(y)\}\).

We use the Pearson chi-square convention
\[
    \chi^2(P\|Q)
    :=
    \int \left(\frac{dP}{dQ}-1\right)^2\,dQ,
\]
with \(\chi^2(P\|Q)=+\infty\) when \(P\) is not absolutely continuous with respect to \(Q\).

\begin{theorem}[Retained \(\mathcal H_\kappa\)-tokens can induce arbitrarily large prefix chi-square divergence]
\label{thm:retained-H-chi2}
Let
\[
    m_{\mathcal H}
    :=
    \nu_\beta(\mathcal U_{\mathcal H}),
    \qquad
    q_{\mathcal H}
    :=
    \nu_\theta^{\mathrm{pre}}(\mathcal U_{\mathcal H})
    =
    \mathbb E_{\nu_\beta}
    \left[
        R_{t-1}(\theta)\mathbf 1\{(x,y,t)\in\mathcal U_{\mathcal H}\}
    \right].
\]
Assume \(m_{\mathcal H}>0\) and
\[
    q_{\mathcal H} \le r\,m_{\mathcal H}
    \qquad\text{for some } r\in(0,1).
\]
Then
\[
    \chi^2\!\left(\nu_\beta\middle\|\nu_\theta^{\mathrm{pre}}\right)
    \ge
    \chi^2\!\left(
        \mathrm{Bern}(m_{\mathcal H})
        \middle\|
        \mathrm{Bern}(q_{\mathcal H})
    \right)
    =
    \frac{(m_{\mathcal H}-q_{\mathcal H})^2}{q_{\mathcal H}(1-q_{\mathcal H})}.
\]
Consequently,
\[
    \chi^2\!\left(\nu_\beta\middle\|\nu_\theta^{\mathrm{pre}}\right)
    \ge
    \frac{m_{\mathcal H}(1-r)^2}{r}.
\]
In particular, for fixed \(m_{\mathcal H}>0\), the lower bound diverges as \(r\downarrow 0\). If \(q_{\mathcal H}=0<m_{\mathcal H}\), then
\[
    \chi^2\!\left(\nu_\beta\middle\|\nu_\theta^{\mathrm{pre}}\right)
    =
    +\infty .
\]
\end{theorem}

\begin{proof}
Let \(P=\nu_\beta\), \(Q=\nu_\theta^{\mathrm{pre}}\), and
\(U=\mathcal U_{\mathcal H}\). If \(Q(U)=0<P(U)\), then \(P\) is not absolutely continuous
with respect to \(Q\), so \(\chi^2(P\|Q)=+\infty\).

Now assume \(q_{\mathcal H}=Q(U)>0\). Let \(L=dP/dQ\). By Cauchy--Schwarz,
\[
    \int_U L^2\,dQ
    \ge
    \frac{\left(\int_U L\,dQ\right)^2}{Q(U)}
    =
    \frac{P(U)^2}{Q(U)}
    =
    \frac{m_{\mathcal H}^2}{q_{\mathcal H}}.
\]
Applying the same argument to \(U^c\),
\[
    \int_{U^c} L^2\,dQ
    \ge
    \frac{P(U^c)^2}{Q(U^c)}
    =
    \frac{(1-m_{\mathcal H})^2}{1-q_{\mathcal H}}.
\]
Therefore,
\[
\begin{aligned}
    \chi^2(P\|Q)
    &=
    \int L^2\,dQ -1 \\
    &\ge
    \frac{m_{\mathcal H}^2}{q_{\mathcal H}}
    +
    \frac{(1-m_{\mathcal H})^2}{1-q_{\mathcal H}}
    -1 \\
    &=
    \frac{(m_{\mathcal H}-q_{\mathcal H})^2}{q_{\mathcal H}(1-q_{\mathcal H})}.
\end{aligned}
\]
Using \(q_{\mathcal H}\le r m_{\mathcal H}\), we have
\[
    m_{\mathcal H}-q_{\mathcal H} \ge m_{\mathcal H}(1-r),
    \qquad
    q_{\mathcal H}(1-q_{\mathcal H})\le q_{\mathcal H} \le r m_{\mathcal H}.
\]
Thus
\[
    \chi^2(P\|Q)
    \ge
    \frac{m_{\mathcal H}^2(1-r)^2}{r m_{\mathcal H}}
    =
    \frac{m_{\mathcal H}(1-r)^2}{r}.
\]
This proves the claim.
\end{proof}

\begin{corollary}[The blow-up is compatible with non-tiny local ratios]
\label{cor:non-tiny-local-ratio}
Fix any threshold \(\tau_c\in(0,1]\) and any target \(M>0\). There exist behavior and current policies \(\beta\) and \(\pi_\theta\), and a retained harmful suffix set \(\mathcal U_{\mathcal H}\), such that every retained suffix token in \(\mathcal U_{\mathcal H}\) satisfies
\[
    \rho_t(\theta)\ge \tau_c,
\]
but
\[
    \chi^2\!\left(\nu_\beta\middle\|\nu_\theta^{\mathrm{pre}}\right)
    >
    M.
\]
\end{corollary}

\begin{proof}
Consider one prompt and two token positions. Let the first-token vocabulary contain two tokens \(c\) and \(b\). Choose \(m\in(0,1)\), and choose
\[
    r < \min\{\tau_c,\,1/2,\,\lambda_2 m/(4M)\},
\]
where \(\lambda_2>0\) is the position weight assigned to the second token.

Let
\[
    \beta(c\mid x)=m,
    \qquad
    \pi_\theta(c\mid x)=rm,
\]
with the remaining probability assigned to \(b\). After prefix \(c\), let both policies deterministically emit the second token \(h\):
\[
    \beta(h\mid x,c)=1,
    \qquad
    \pi_\theta(h\mid x,c)=1.
\]
Assign negative advantage to the trajectory \(y=(c,h)\), i.e., \(A(y)<0\).

For \(y=(c,h)\), the first token has
\[
    \rho_1(\theta)
    =
    \frac{\pi_\theta(c\mid x)}{\beta(c\mid x)}
    =
    r
    <
    \tau_c,
\]
so \(\kappa(y)=1\). The second token has
\[
    \rho_2(\theta)
    =
    \frac{\pi_\theta(h\mid x,c)}{\beta(h\mid x,c)}
    =
    1
    \ge
    \tau_c,
\]
so \(2\in \mathcal H_\kappa(y)\). Its prefix occupancy is
\[
    R_1(\theta)=\rho_1(\theta)=r.
\]
Let the update rule retain this second-token occurrence, and define
\[
    \mathcal U_{\mathcal H}=\{(x,(c,h),2)\}.
\]
Then
\[
    m_{\mathcal H}=\nu_\beta(\mathcal U_{\mathcal H})=\lambda_2 m,
    \qquad
    q_{\mathcal H}=\nu_\theta^{\mathrm{pre}}(\mathcal U_{\mathcal H})=\lambda_2 r m.
\]
By Theorem~\ref{thm:retained-H-chi2},
\[
    \chi^2\!\left(\nu_\beta\middle\|\nu_\theta^{\mathrm{pre}}\right)
    \ge
    \frac{\lambda_2 m(1-r)^2}{r}.
\]
Since \(r<1/2\), \((1-r)^2>1/4\), and therefore
\[
    \chi^2\!\left(\nu_\beta\middle\|\nu_\theta^{\mathrm{pre}}\right)
    >
    \frac{\lambda_2 m}{4r}
    >
    M.
\]
Thus the retained token can have a perfectly benign local ratio \(\rho_2(\theta)=1\), while the induced prefix mismatch is arbitrarily large.
\end{proof}

\begin{remark}[Direction of the Pearson divergence]
\label{rem:chi2-direction}
The divergence in Theorem~\ref{thm:retained-H-chi2} is the reverse Pearson divergence \(\chi^2(\nu_\beta\|\nu_\theta^{\mathrm{pre}})\), which is sensitive to behavior-supported prefixes that have very small current occupancy. The opposite direction satisfies
\[
    \chi^2\!\left(\nu_\theta^{\mathrm{pre}}\middle\|\nu_\beta\right)
    =
    \mathbb E_{\nu_\beta}
    \left[
        \left(R_{t-1}(\theta)-1\right)^2
    \right],
\]
so \(R_{t-1}\ll 1\) on a set of fixed behavior mass does not by itself force an unbounded contribution in that direction. This is why the missing-support statement should be phrased as a behavior-to-current, or reverse-Pearson, prefix divergence.
\end{remark}
\section{More Related Work}
\label{app:more_related_work}

\paragraph{Off-policy RLVR.}
Despite being framed as on-policy methods, practical RLVR implementations are inherently off-policy: rollout data are reused across multiple gradient updates~\citep{shao2024deepseekmathpushinglimitsmathematical, openr1}; modern RL systems decouple rollout generation from training for throughput~\citep{fu2026areal, slime_github, noukhovitch2025asynchronous, zhong2025streamrl, he2025history}; and training--inference mismatch from numerical precision differences can introduce additional distribution shift~\citep{qi2025defeating}. To stabilize training under off-policy conditions, prior work has proposed various strategies to control policy updates, including ratio clipping~\citep{yu2026dapo}, approximate trust regions~\citep{fu2026areal}, sequence-level clipping~\citep{zheng2025group}, asymmetric clipping~\citep{Roux2025TaperedOR}, and gradient-preserving clipping~\citep{su2025klear, chen2025minimax}. While these methods improve RLVR under moderate staleness, most focus on relatively small staleness values and do not explicitly identify why naive large-staleness GRPO fails. M2PO~\citep{zheng2026prosperity} is the closest work to ours. It identifies excessive clipping under staleness and the prosperity-before-collapse phenomenon, where training without a trust region initially improves before eventually collapsing. M2PO addresses these issues by constraining the second moment of importance weights at the batch level to selectively mask high-variance tokens. In contrast, we diagnose the failure mode of large-$\mu$ GRPO as localized negative-advantage suffix drift and stabilize training with a simple trigger--scope filter, without requiring batch-level second-moment statistics or iterative token masking. Moreover, our work introduces an off-policy training paradigm that fully decouples generation and optimization to fundamentally address the efficiency bottleneck, an aspect M2PO does not explore.

\paragraph{RLVR Efficiency.}
RLVR systems improve efficiency through both synchronous/hybrid orchestration and asynchronous execution. In synchronous pipelines, rollout generation and policy optimization often alternate between inference and training engines. To make frequent switching practical, systems may keep part of the training state resident on GPU during rollout, reducing reload overhead but limiting memory available for high-throughput inference. Hybrid orchestration frameworks such as HybridFlow/verl improve scheduling across these engines~\citep{sheng2025hybridflow, slime_github}, but frequent rollout refreshes and weight synchronization can remain a bottleneck when the RL algorithm only tolerates low staleness. Asynchronous and streaming systems instead overlap rollout workers and trainers through queues or streamed data transfer~\citep{fu2026areal, zhong2025streamrl, han2025asyncflow, noukhovitch2025asynchronous}. These systems improve utilization, but still benefit from staleness-robust objectives because larger tolerated rollout lag reduces synchronization pressure. \ourmethod{} is complementary to both directions: by stabilizing GRPO at $\mu=1024$, it reduces rollout refreshes and weight synchronizations, enabling rollout generation to run as a standalone phase that can fully use GPU memory for inference, while also permitting larger lag in asynchronous pipelines.

\end{document}